\documentclass[preprint]{article}

\PassOptionsToPackage{numbers, compress}{natbib}
\usepackage{nips_2018}

\usepackage[utf8]{inputenc} 
\usepackage[T1]{fontenc}    
\usepackage{hyperref}       
\usepackage{url}            
\usepackage{booktabs}       
\usepackage{amsfonts}       
\usepackage{dsfont}
\usepackage{nicefrac}       
\usepackage{microtype}      
\usepackage{todonotes}
\usepackage{amsmath,amssymb}
\usepackage{amsthm}

\newtheorem{theorem}{Theorem}

\DeclareMathOperator*{\arginf}{\mathrm{arg inf}}
\newcommand{\E}{\mathbb{E}}

\newcommand{\divergence}[2]{D\!\left(#1 \middle\| #2 \right)}

\newcommand{\KL}[2]{D_{KL}\!\left(#1 \middle\| #2 \right)}
\newcommand{\ELBO}[2]{\text{ELBO}\!\left(#1, #2 \right)}

\newcommand{\mean}[2]{\E_{#2} \! \left[ #1 \right]}
\newcommand{\varder}[2]{\nabla_#1 #2}
\renewcommand{\d}[1]{\ensuremath{\operatorname{d}\!{#1}}}

\title{Forward Amortized Inference for Likelihood-Free Variational Marginalization}

\author{
 Luca Ambrogioni \\
 Radboud University\\
 \texttt{l.ambrogioni@donders.ru.nl} \\
 \And
 Umut Güçlü \\
 Radboud University\\
 \texttt{u.guclu@donders.ru.nl} \
 \And
 Julia Berezutskaya  \\
 University of Utrecht\\
 \texttt{y.berezutskaya@umcutrecht.nl} \\
 \And
 Eva W. P. van den Borne \\
 Radboud University\\
 \texttt{e.vandenborne@student.ru.nl} \
 \And
 Yağmur Güçlütürk \\
 Radboud University\\
 \texttt{y.gucluturk@donders.ru.nl} \
 \And
 Max Hinne \\
 University of Amsterdam\\
 \texttt{m.hinne@uva.nl} \\
 \And
 Eric Maris \\
 Radboud University\\
 \texttt{e.maris@donders.ru.nl} \\
 \And
 Marcel A. J. van Gerven \\
 Radboud University\\
 \texttt{m.vangerven@donders.ru.nl} \\
}

\begin{document}

\maketitle
\begin{abstract}
In this paper, we introduce a new form of amortized variational inference by using the forward KL divergence in a joint-contrastive variational loss. The resulting forward amortized variational inference is a likelihood-free method as its gradient can be sampled without bias and without requiring any evaluation of either the model joint distribution or its derivatives. We prove that our new variational loss is optimized by the exact posterior marginals in the fully factorized mean-field approximation, a property that is not shared with the more conventional reverse KL inference. Furthermore, we show that forward amortized inference can be easily marginalized over large families of latent variables in order to obtain a marginalized variational posterior. We consider two examples of variational marginalization. In our first example we train a Bayesian forecaster for predicting a simplified chaotic model of atmospheric convection. In the second example we train an amortized variational approximation of a Bayesian optimal classifier by marginalizing over the model space. The result is a powerful meta-classification network that can solve arbitrary classification problems without further training. 


\end{abstract}
\section{Introduction} 
Bayesian inference is a principled statistical framework for estimating the probability of latent factors given a set of observations. Unfortunately, most complex Bayesian models are intractable since computing the posterior distribution involves the solution of integrals over high-dimensional spaces. Variational inference (VI) is a family of approximation methods that reframes Bayesian inference as an optimization problem that can be solved using stochastic optimization techniques \cite{jordan1999introduction}. Recent developments in stochastic VI have scaled Bayesian inference to massive datasets and paved the way for the integration of deep learning and Bayesian statistics \cite{hoffman2013stochastic, ranganath2014black, rezende2014stochastic, kucukelbir2017automatic, tran2016edward}. In many applications, VI is made more efficient by optimizing a whole family of variational distributions at once \cite{kingma2013auto, huszar2017variational, ritchie2016deep}. This approach is usually referred to as amortized inference. Amortized inference can be seen as a special case of the larger framework of joint-contrastive variational inference \cite{huszar2017variational, dumoulin2016adversarially}. 

In this paper we introduce forward amortized variational inference (FAVI) as a flexible and tractable new form of likelihood-free VI. FAVI is obtained by using the forward KL divergence on a joint-contrastive variational loss. One of the most important features of FAVI is that it can be used for marginalizing over a large space of nuisance variables without explicitly modeling their joint density. Marginalization of nuisance variables is important in many real-world problems such as weather forecasting \cite{gneiting2005weather}. FAVI is particularly suitable for model-based problems such as weather forecasting because it is trained on samples from the generative model. However, the applicability of FAVI goes far beyond model-based problems. As an example of a model-free problem, we use FAVI to obtain a meta-classifier as a variational approximation of the Bayes optimal classifier of an infinite ensemble of classification models. The resulting variational meta-classifier is algorithmically similar to the meta-learning methods introduced in \cite{prokhorov2002adaptive} and recently expanded in \cite{santoro2016meta,vinyals2016matching}. 

\section{Related work}
In spite of its theoretical advantages, the intractability of the expectation in the forward KL divergence $\KL{p(z|x)}{q(z)}$ limits its applicability in the conventional VI framework \cite{christopher2016pattern}. The forward KL is adopted by expectation propagation (EP) methods \cite{minka2001expectation, barthelme2011abc}, but EP is not a form of VI since it does not minimize a global divergence between the two distributions. Likelihood-free Bayesian inference is often based on approximate Bayesian computation (ABC) \cite{tavare1997inferring, pritchard1999population}. Recently the ABC approach has been applied to both VI \cite{tran2017variational} and EP \cite{barthelme2011abc}. However, despite its success in many applications, ABC has some important limitations. In particular, the efficiency of rejection based ABC methods tends to sharply degrade  as the dimensionality grows and the use of low-dimensional summary statistics can severely affect the performance. Similarly, methods based on some form of density estimation such as \cite{shi2018kernel} are strongly affected by the curse of dimensionality since high-dimensional density estimation is notoriously challenging. An alternative approach, which is algorithmically similar to our method, is to treat Bayesian inference as a nonlinear regression problem. This approach was first introduced in \cite{blum2010non} and recently extended in \cite{papamakarios2016fast}. In this latter work, a loss similar to our FAVI loss was iteratively optimized using an importance sampling scheme so that the simulator ($p(z,x)$ in our notation) gradually narrows down to the distribution of the observed data. Note that this work does not draw any connection with VI and their importance sampling scheme is explicitly designed to avoid inference amortization. In general, the FAVI approach offers a theoretical foundation to several previous works based on training deep networks on simulated data \cite{le2017using, jaderberg2014synthetic, jaderberg2016reading, gupta2016synthetic, stark2015captcha, guccluturk2016convolutional, ambrogioni2017estimating}. Most of the recent literature about likelihood-free approximate Bayesian inference is based on adversarial training. This line of research was initiated by adversarially learned inference (ALI) which can be shown to minimize the Jensen-Shannon divergence at the limit of an optimal discriminator \cite{dumoulin2016adversarially}. Several other adversarial VI methods have recently been introduced \cite{mescheder2017adversarial, tran2017hierarchical, huszar2017variational}. These variational methods share some of the flexibility of FAVI, but they usually require the samples from $p$ to be differentiable. A drawback of adversarial methods is that the adversarial minimax problem is equivalent to the minimization of a divergence only in the nonparametric limit \cite{goodfellow2014generative, mescheder2017adversarial}. From a practical perspective, variational methods tend to generate very realistic samples, but often suffer from instability during training and mode collapse \cite{arora2018do, arora2017generalization}. 

\section{Background on joint-contrastive variational inference}
Joint-contrastive variational inference was first introduced in the context of ALI \cite{dumoulin2016adversarially} and more explicitly outlined in \cite{huszar2017variational}. The loss functional of joint-contrastive variational inference is a divergence between the model joint distribution and a joint variational distribution:
\begin{equation}\label{eq: joint divergence}
\mathcal{L}_{j}[p,q] = \divergence{p(z,x)}{q(z,x)}~.
\end{equation}
Without further constraints the minimization of this loss functional is not particularly useful as the model joint $p(z,x)$ is usually tractable and it does not need to be approximated. The key idea for approximating the intractable posterior $p(z|x)$ by minimizing \ref{eq: joint divergence} is to factorize the variational joint as the product of a variational posterior $q(z|x)$ and the sampling distribution of the data:
\begin{equation}\label{eq: joint variational distribution}
q(x,z) = q(z|x)k(x)~.
\end{equation}
Usually $k(x)$ is a re-sampling distribution of a training set as in the case of variational autoencoders~\cite{kingma2013auto}. Given this factorization, the minimization of \ref{eq: joint divergence} with respect to both $q$ and $p$ simultaneously approximates the model posterior with $q(z|x)$ and the real-word distribution with $p(x)$. Importantly, we can usually sample from both $q(x,z)$ and $p(z,x)$ and this implies that we can stochastically optimize \ref{eq: joint divergence} for a large class of divergence measures. 

\subsection{Amortized inference}
If we adopt the KL divergence in Eq.~\ref{eq: joint variational distribution}, the joint-contrastive variational inference loss decomposes into an evidence loss and an amortized inference loss term:
\begin{equation}\label{eq: reverse KL}
\KL{q(x,z)}{p(x,z)} = \KL{k(x)}{p(x)} + \mean{\KL{q(z|x)}{p(z|x)}}{k(x)}~.
\end{equation}
The result suggests that conventional amortized inference is a special case of joint-contrastive variational inference. We can see this by studying the gradients of Eq.~\ref{eq: reverse KL}. In the following, $\varder{q}{}$ denotes the functional gradient with respect to the density $q$. We use this functional notation in order to avoid referring to an explicit parametrization. Since the term corresponding to the entropy of $k(x)$ in Eq.~\ref{eq: reverse KL} does not depend on $q$, this divergence has the same functional gradient as the (negative) amortized ELBO:
\begin{align}\label{eq: ELBO gradient q}
\varder{q}{\KL{q}{p}} &= \varder{q}{\mean{\log{\frac{q(z|x)}{p(x,z)}}}{q(x,z)}} + \varder{q}{\mean{\log{k(x)}}{k(x)}} \notag \\
&= -\varder{q}{\mean{\ELBO{q}{p}}{k(x)}}~.
\end{align}
Therefore, amortized variational inference is a special case of joint-contrastive variational inference.

\section{Forward amortized variational inference}
The reverse KL divergence has a central position in the classical (posterior-contrastive) variational framework because it leads to a tractable variational lower bound. Conversely, the forward KL divergence is intractable in a posterior-contrastive sense as it requires computation of an expectation with respect to the true posterior. We will now show that the forward KL is tractable when used in a joint-contrastive loss. In this case we obtain the following divergence:
\begin{align} \label{eq: FAVI divergence}
\KL{p(x,z)}{q(x,z)} &= \mean{\log{\frac{p(x,z)}{q(z|x) k(x)}}}{p(x,z)} \notag\\ 
&= -\mean{\log{q(z|x)}}{p(x,z)} +\mean{\log{\frac{p(x,z)}{k(x)}}}{p(x, z)}~.
\end{align}
Note that in this expression there is only one term that depends on $q$. Therefore, by ignoring the constant terms, we can define the FAVI loss as follows:
\begin{equation}\label{eq: FAVI loss}
\mathcal{L}_{FA} = -\mean{\log{q(z|x)}}{p(x,z)}~.
\end{equation}
The resulting functional gradient is given by
\begin{equation}\label{eq: FAVI gradient}
\varder{q}{\mathcal{L}_{FA}} = -\mean{\varder{q}{\log{q(z|x)}}}{p(x,z)}~.
\end{equation}
Note that the computation of this gradient requires neither reparametrization tricks nor black-box methods, since the expectation is taken with respect to $p$ while the gradient is taken with respect to $q$.

The FAVI variational loss can also be derived as an amortized form of posterior-contrastive variational inference. The forward KL posterior-contrastive variational loss is given by:
\begin{equation}\label{eq: FAVI amortized derivation}
\KL{p(z|x)}{q(z|x)} = \mean{\log{\frac{p(z|x)}{q(z|x)}}}{p(z|x)}~.
\end{equation}
It is challenging to obtain unbiased samples from the gradient of this expression as the expectation is taken with respect to the intractable $p(z|x)$. We can recover the FAVI loss (up to a term constant in $q$) if we amortize the problem with respect to the model probability:
\begin{equation}
\mean{\KL{p(z|x)}{q(z|x)}}{p(x)} = -\mean{\log{q(z|x)}}{p(x,z)} + \mean{\log{p(x,z)}}{p(x, z)} ~.
\end{equation}
FAVI has several advantages over reverse amortized inference. First of all, it is very simple to obtain Monte Carlo samples of the gradients of the stochastic loss in Eq.~\ref{eq: FAVI loss}, since the expectation is taken with respect to $p$. This avoids the use of methods such as the reparametrization trick, which limits the family of possible probability distributions and lengthens the computational graph since the loss needs to be back-propagated through the samples \cite{rezende2014stochastic}. Another important advantage is that the model joint probability $p(x,z)$ does not need to be evaluated explicitly. This implies that FAVI can be used when the likelihood is intractable, in situations where ABC methods are usually adopted \cite{csillery2010approximate, blum2010non, marin2012approximate, tran2017variational}. A downside of FAVI is that Eq.~\ref{eq: FAVI divergence} cannot be directly minimized with respect to $p$ since $k(x)$ cannot be expressed in closed form. There are several possible ways for dealing with this problem. In Appendix A we outline an adversarial method that only requires the differentiability of the samples from $p$. Note that the optimization of $p$ is not strictly speaking part of Bayesian inference. Therefore we will focus on the case where the generative model $p$ is known \emph{a priori} in the rest of the paper. One of the most interesting features of FAVI is that its loss is optimized by the exact marginals even when the variational approximation is fully factorized, as we shall demonstrate in the next section.

\subsection{Marginalization properties of FAVI}
In the fully factorized mean field approximation the FAVI loss is minimized by the exact marginals of the true posterior, as stated in the following theorem:
\begin{theorem}[Exact marginals] \label{th: exact marginals}
Consider a joint distribution $p(z, x)$ and a fully factorized variational posterior $q(z|x) = \prod_k q_k(z_k|x)$. The functional $\mathcal{L}_{FA}[p,q]$ is minimized when $q(z|x) = \prod_k p(z_k|x)$ for all $x$ in the support of $p(x)$. Furthermore, the minimizer is unique when all values of $x$ are in the support of $p(x)$.
\end{theorem}
 \begin{proof}
 In the fully factorized case, the FAVI loss can be rewritten as follows:
\begin{align} \label{eq: proof I}
\mathcal{L}_{FA} & = - \mean{ \sum_k\log{q_k(z_k|x)}}{p(x,z)} \notag\\
 & = -\sum_k \mean{\log{q_k(z_k|x)}}{p(z_k|x)p(x)} \notag\\
 & = \sum_k \mean{\KL{p(z_k|x)}{q_k(z_k|x)}}{p(x)} - \sum_k \mean{\log{(p(z_k|x))}}{p(z_k|x)}~.
\end{align}
The conditional entropy term on the right side of the final expression does not depend on $q$ and can therefore be ignored. Since the KL divergence is always non-negative and vanishes only when the two distributions are identically equal, the expectations in the remaining term are equal to zero if and only if $q_k(z_k|x) = p(z_k|x)$ for all $k$ and for all $x$ in the support of $p(x)$.
\end{proof}
The situation is radically different in reverse KL VI where the factorized approximation can lead to a severe underestimation of the uncertainty of the marginals \cite{robert2014machine, christopher2016pattern}. 

Theorem \ref{th: exact marginals} straightforwardly generalizes to variational models that are factorized into two blocks. From this, an important result follows:
\begin{theorem}[Consistent marginalization]
\label{th: consistent marginalization}
Consider a joint distribution $p(z, \xi, x)$ and the (nonparametric) conditionally independent variational model $q(z, \xi|x) = q_{z}(z|x) q_{\xi}(\xi|x)$. The following equality holds:
\begin{equation} \label{eq: theorem I}
\int \arginf_{q} \mathcal{L}_{FA}[p(z,\xi,x),q(z,\xi|x)] d\xi = \arginf_{q_z} \mathcal{L}_{FA}[p(z,x),q_z(z|x)]~.
\end{equation}
\end{theorem}
\begin{proof}
\begin{align}
\int \arginf_{q} \mathcal{L}_{FA}[p(z,\xi,x),q(z,\xi|x)] \d\xi &=  \int p(z| x) p(\xi| x) \d\xi \notag \\
& = p(z| x) = \arginf_{q_z} \mathcal{L}_{FA}[p(z,x),q_z(z|x)]~,
\end{align}
where the first equality is a direct consequence of Theorem \ref{th: exact marginals}.
\end{proof}
Therefore, there is no need to explicitly model the conditional dependencies between $z$ and $\xi$ when the aim is to estimate $q_z(z|x)$. In practice, it is straightforward to obtain Monte Carlo estimates of the marginalized variational loss $\mathcal{L}_{FA}[p(z,x),q_z(z|x)]$, since a sample from $p(z,x)$ is obtained by `ignoring' $\xi$ from a sample from the full joint distribution. Conversely, marginalization in the reverse KL approach requires to either perform the challenging integration $p(z,x) = \int p(z,\xi,x) d\xi$ or to explicitly model the conditional dependencies between $z$ and $\xi$ and marginalize out $\xi$ from the resulting variational distribution. Note that Theorem \ref{th: consistent marginalization} does not hold in the case of reverse KL inference and, consequently, assuming conditional independence could severely bias the resulting marginalized posterior.

\section{Applications}
We begin this section with a direct comparison between amortized reverse VI and FAVI. In this comparison we approximate the variational posterior of a variational autoencoder and we compare the accuracy of the two variational posteriors. Subsequently, we discuss two applications where the reverse KL approach is not easily applicable. These applications involve large-scale likelihood-free marginalization of latent variables. In the first application we use FAVI to obtain a variational forecaster of chaotic time series. This is an example of a model-based problem since the dynamic equations are assumed to reliably describe the dynamics of real-world systems such as the earth's atmosphere. In the second application we apply FAVI to the model-free problem of meta-classification. In this case, predictive performance is obtained by marginalizing over the posterior distribution of a weakly structured ensemble of random classification models that span a very large space of possible classification problems.

\subsection{Comparison between amortized inference methods}
In order to compare FAVI with other amortized inference approaches, we approximated the posterior distribution $p(\boldsymbol{z}|\boldsymbol{x})$ of the generative model
$$
p(\boldsymbol{x},\boldsymbol{z}) = \mathcal{N}\!\big(\boldsymbol{x}|\boldsymbol{f}_\mu(\boldsymbol{z}),\text{diag}(\boldsymbol{f}_\sigma(\boldsymbol{z}))\big) \mathcal{N}\!\big(\boldsymbol{z}|\boldsymbol{0},I\big)~,
$$
where $\boldsymbol{x}$ is a vector containing the intensity of the pixels of a black-and-white image and $\boldsymbol{z}$ is a vector of latent variables. The functions $\boldsymbol{f}_\mu(\boldsymbol{z})$ and $\boldsymbol{f}_\sigma(\boldsymbol{z})$ are the two outputs of a pre-trained deep neural network. The network has a three-layered fully connected architecture with ReLu nonlinearities in the hidden layers and was trained on the MNIST dataset using a variational autoencoder \cite{kingma2013auto}. We decided to use a common pre-trained generator in order to have a clean comparison between the performances of the approximate Bayesian inference methods. The variational posterior $q(\boldsymbol{z}|\boldsymbol{x})$ was parametrized by a three-layered fully connected architecture with ReLu nonlinearities. Both models trained with Adam \cite{kingma2014adam} for $100$ epochs with batch size $200$. The reverse KL inference network was trained by re-sampling MNIST images while FAVI was trained on simulated samples. This difference follows from the fact that the former is amortized with respect to $k(x)$ while the latter is amortized with respect to $p(x)$. We also included ALI \cite{dumoulin2016adversarially} in this comparison as an example of an adversarial likelihood-free method. 

\subsubsection{Results}
\begin{figure}[!t]
    \centering
    \includegraphics[width=1\textwidth]{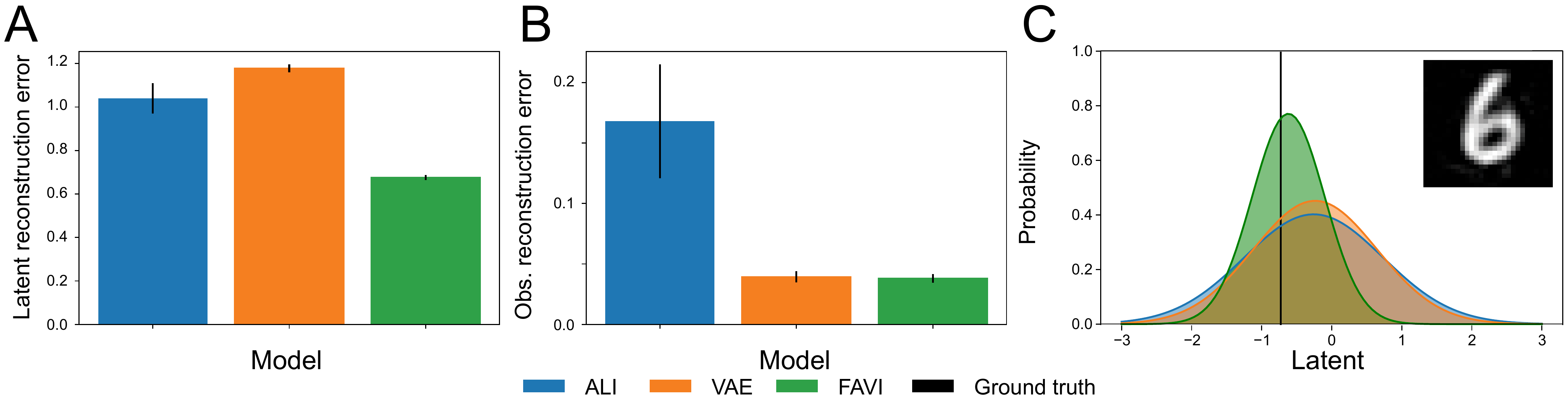}
    \caption{Comparison of the performance of the variational inference methods on the MNIST dataset. \textbf{A.} Reconstruction error of the latent variables. \textbf{B.} Reconstruction error of the images. \textbf{C.} Example of variational distributions given the (synthetic) image shown in the upper right corner. The black line denotes the real value of the latent variable.}
    \label{fig:autoencoder}
\end{figure}

The latent reconstruction error was quantified as respectively
\begin{equation}
\mean{\frac{1}{N}\sum_{j=1}^N(z_j - \hat{z}_j)^2}{q(\hat{\boldsymbol{z}}|\boldsymbol{x}) p(\boldsymbol{x},\boldsymbol{z})} ~~ \text{and} ~~ \mean{\frac{1}{M}\sum_{j=1}^M(x_j - \hat{x}_j)^2}{p(\hat{\boldsymbol{x}}|\hat{\boldsymbol{z}})q(\hat{\boldsymbol{z}}|\boldsymbol{x})k(\boldsymbol{x})}~,
\end{equation}
where $N$ is the dimension of the latent space and $M$ is the number of pixels. We tested the statistical difference between the errors using two-sample t-tests. Figure~\ref{fig:autoencoder}A shows the reconstruction error of the latent variable given a generated image. As we can see, FAVI has a remarkably lower latent reconstruction error when compared with reverse KL VI (p < 0.001). The latent error of ALI is slightly smaller than the error of reverse KL VI (p < 0.001). The superior performance of FAVI could have been expected since FAVI is trained on generated images while the reverse KL method is trained directly on real data. However, FAVI also has a slightly lower and less variable observable reconstruction error (p < 0.05). This can be seen in Fig.~\ref{fig:autoencoder}B. Conversely, ALI has a very high reconstruction error. Figure~\ref{fig:autoencoder}B shows the two variational distributions of the first component of the latent vector for an example image.

\subsection{Bayesian variational forecaster}
Forecasting the future of a dynamical system based on past noisy measurements and a system of dynamic equations is crucial for many scientific applications \cite{west1996bayesian}. The most well-known of these applications is arguably weather forecasting \cite{gneiting2005weather}. FAVI is particularly appropriate for dynamic forecasting problems for three main reasons. First, in these problems the generator is known with good accuracy and this benefits approaches like FAVI where the training samples are sampled from the generator. Second, it is often difficult to obtain analytic expressions for the probability densities of the dynamic and the noise models. Third, forecasting highly benefits from the marginalization of nuisance variables and unknown parameters \cite{gneiting2005weather}. We validated our FAVI forecaster on a simulated dataset. We generated chaotic time series using a very simplified model of atmospheric convection: the Lorentz dynamical system \cite{lorenz1963deterministic}. The system is given by the following differential equations:
\begin{align*}
\dot{x}_1(t) &= 10~\big(x_2(t) - x_1(t)\big) \\
\dot{x}_2(t) &= x_1(t) \big(28 - x_3(t)\big) - x_2(t) \\
\dot{x}_3(t) &= x_1(t) x_2(t) - 8/3 x_3(t)~,
\label{eq: lorentz dynamical system} 
\end{align*}
where the dot denotes a derivative with respect to time. In our case, the task is to estimate the probability of the value of $x_1$ at the future time point $t^*$ given a set of $M$ noise-corrupted observations $Y = \{y(t_0),...,y(t_M)\}$ where 
\begin{equation}
y(t) \sim \mathcal{N}\!(x_1(t), 10^2)~.
\end{equation}
Note that the variables $x_2$ and $x_3$ are not observed and need to be marginalized out. The graphical model of the complete and marginalized joint is given in Fig.~\ref{fig:forecaster}A. The FAVI loss is given by:
\begin{equation}\label{eq: forecaster FAVI loss}
\mathcal{L}_{FA} = -\mean{\log{q\left(x(t^{*})|Y\right)}}{p\left(x(t^{*}),Y\right)}~.
\end{equation}
We parametrized $q\left(x(t^{*})|Y\right)$ using a dilated convolutional neural network \cite{yu2015multi} with a kernel mixture network output \cite{ambrogioni2017kernel}, the details of the architecture are given in Appendix B.

\subsubsection{Results}
We compared the performance of our variational Bayesian forecaster against the extended Kalman filter (EKF), one of the most popular off-the-shelf dynamic forecasting methods~\cite{evensen2009data}. Specifically, we used the EKF for obtaining the joint posterior probability density of each variable at the last time point $t_M$ given the observations. By construction of the EKF approximation, this probability is a multivariate normal distribution. We made a forecast by numerically integrating $500$ time series from $t_M$ to $t^{*}$, where the initial conditions were sampled from the EKF posterior density at $t_M$. Figure~\ref{fig:forecaster}B shows the forecast of a randomly sampled example trial together with the ground truth. The predictive distribution of the Bayesian variational forecaster is tightly tracking the ground truth. Interestingly, the variational posterior bifurcates into the two possible `wings' of the Lorentz attractor. For each validation trial the performances of the EKF and the variational Bayesian forecaster were quantified as the probability of $x_1(t^{*})$ being inside a symmetric interval centered around the ground truth with radius $3$. In the EKF case this probability was obtained by counting the number of samples inside the interval and dividing by the total number of samples, while in the case of the Bayesian variational forecaster the probability was obtained by integrating the variational posterior probability density inside the interval. Figure~\ref{fig:forecaster}C shows the scatter plot of these probabilities for 500 validation trials. On average the performance of the variational Bayesian forecaster is $1.94$ times higher than the performance of the EKF. 

\begin{figure}[!t]
	\includegraphics[width=\textwidth]{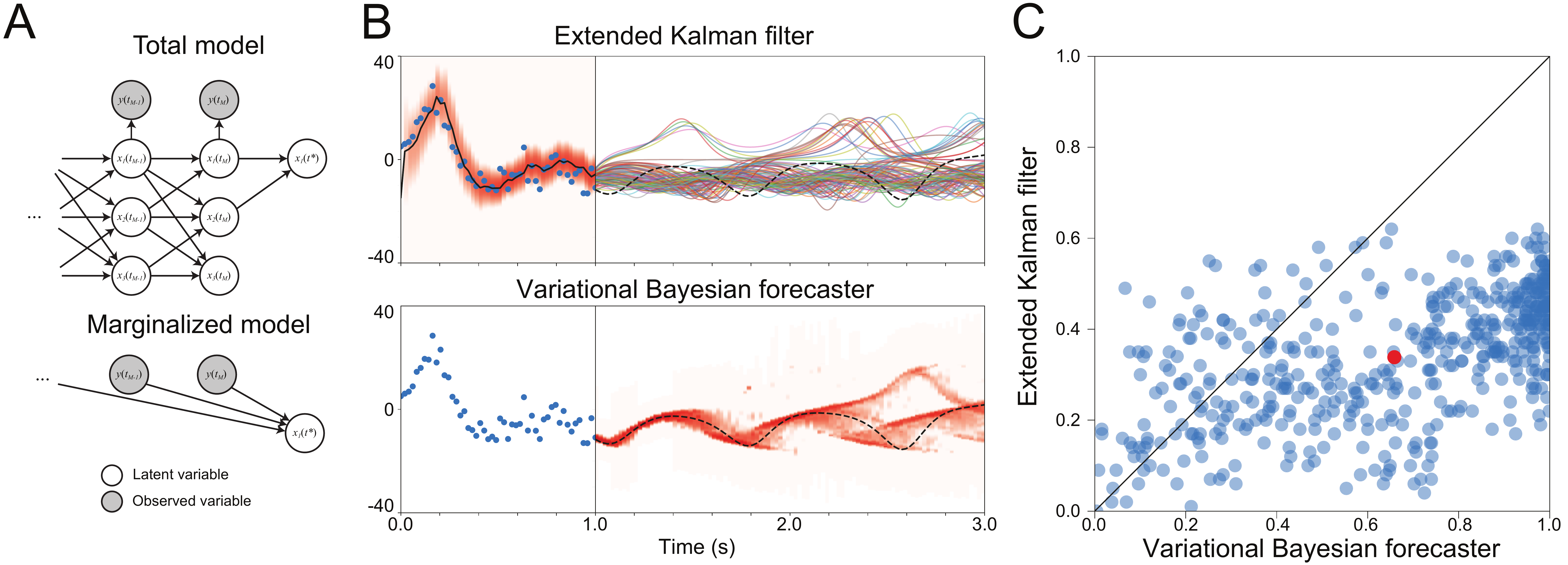}
	\caption{\textbf{A.} The total and marginalized generative models for forward autoencoders. \textbf{B.} EKF (top panel) and variational (bottom panel) forecast of a time series sampled from the Lorentz system. The blue dots denote the noise-corrupted observations. \textbf{C.} Forecast of a Lorentz dynamical system. The blue dots are individual simulated trials and the red dot denote the mean (center of mass).}
    \label{fig:forecaster}
\end{figure}


\subsection{Bayesian variational meta-classifier}
We now introduce a real-wold application that showcases the flexibility and scalability of FAVI when the real generative model is unknown. Our aim is to construct a Bayesian meta-classifier as an amortized variational approximation of the Bayes optimal classifier of an ensemble. Conventional variational methods are not suited for this task as they would need to introduce a variational distribution over the potentially infinite and unstructured model space and explicitly marginalize over the resulting posterior. Furthermore, the model likelihood $p(D|M_k)$ is very often non-differentiable and even impossible to evaluate in closed form. The lack of differentiability would rule out adversarial variational methods. We begin by giving a brief introduction to ensemble methods and Bayes optimal classifiers.

\subsubsection{Bayesian ensembles}

\begin{figure}[!t]
	\includegraphics[width=\textwidth]{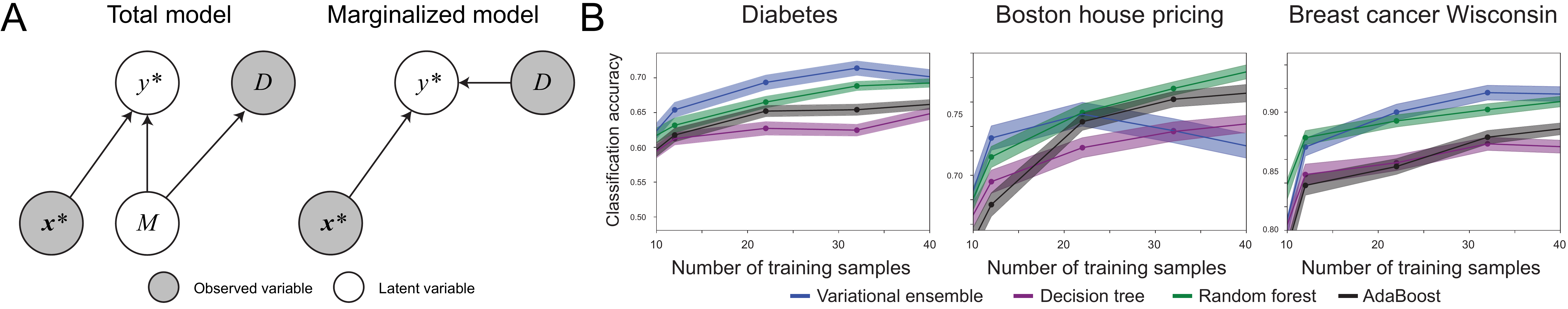}
	\caption{\textbf{A.} The total and marginalized generative models for the Bayesian meta-classifier. \textbf{B.} The classification accuracy of the Bayes variational meta-classifier versus three common alternatives on three data sets.}
    \label{fig:meta-classifier}
\end{figure}

In a classification task the aim is to estimate the probability of the target class assignments $y$ given a set of predictors $\boldsymbol{x}$. In an ensemble learning setting we assume that the classification task is sampled from a predefined family of classification models $M_1, M_2, ..., M_K$. In our notation we consider two models with the same parametric form, but different parameter values as different models. An ensemble classifier has the following form:
\begin{equation}
p(y^*|\boldsymbol{x}^*, D) = \sum_{k=1}^K w_k(D) p(y^*|\boldsymbol{x}^*, M_k)
  \label{eq: ensemble classifier}~, 
\end{equation}
where $x^*$ denotes a new vector of predictors, $y^*$ denotes the corresponding label and $D$ denotes the training data. Different ensemble models use different techniques for setting the weights $w_k(D)$. The optimal way of setting the weights $w_k(D)$ can be obtained formally using Bayes' rule. The posterior probability of each model $M_k$ given the training data is given by:
\begin{equation}
p(M_k|D) = \frac{p(D|M_k) p(M_k)} {p(D)}~,
  \label{eq: Bayesian posterior} 
\end{equation}
where $D$ is a training set of predictors $\boldsymbol{x}$ and target class assignments $y$. Assuming that we know the prior over the family of classification models, the optimal solution to the classification problem is given by marginalizing the posterior distribution $p(y|\boldsymbol{x})$ over all models $M_1, M_2, ..., M_K$ \cite{mitchell1997machine}. This is known as the Bayesian optimal classifier:
\begin{equation}
p(y^*|\boldsymbol{x}^*, D)=\sum_{k=1}^K p(y^*|\boldsymbol{x}^*, M_k) p(M_k|D)
  \label{eq: Bayes optimal classifier}~. 
\end{equation}
In practice, computing the Bayesian optimal classifier is intractable as it involves a sum (or an integral) over the whole (usually infinite) ensemble of models.

\subsubsection{Variational meta-classifier}
A Bayesian variational meta-classifier can be obtained by approximating the Bayesian optimal classifier using FAVI. The model is amortized with respect to whole training sets consisting of feature/label pairs, which are assumed to be generated by one (and only one) of the models in the ensemble. The resulting amortized posterior model is a meta-classifier, as it takes as input a training set and it outputs the predictive distribution over the label of an arbitrary new data-point. The forward amortized loss is given by:
$$
\mathcal{L}_{FA}[q] = -\mean{q(y^*|x^*,D)}{p(y^*,x^*,D)}~,
$$
where
$$
p(y^*,x^*,D) = \sum_k p(y^*,x^*,D| M_k) p(M_k)~.
$$
The graphical models of both the total and the marginalized joint are shown in Fig.~\ref{fig:meta-classifier}A. We trained a RNN using the FAVI loss in order to approximate the predictive distribution $p(x|y)$. Our variational posterior is given by: 
\begin{equation}
q(y^* = 1|\boldsymbol{x}^*) = \textrm{RNN}(\boldsymbol{x}^*;D)~,
\label{eq: approximate posterior ||} 
\end{equation}
where $\textrm{RNN}(\boldsymbol{x}^*; D)$ is a recurrent architecture that has received as input a training set $D$ of training pairs $(\boldsymbol{x},y)$. The details of our RNN architecture are given in Appendix C. 

\subsubsection{Results}
We trained a Bayes variational meta-classifier model on the ensemble of generative models described in Appendix D. The network was trained on binary classification with $10$ predictors. The Chainer deep learning framework \cite{tokui2015chainer} was used for model training. After training the model was tested separately on three public real-world datasets: the Boston house-prices dataset, the diabetes dataset and the breast cancer Wisconsin dataset \cite{harrison1978hedonic, efron2004least, street1993nuclear}. In all datasets only the first $10$ predictors were used. The Boston dataset is a regression problem but we converted it into a classification problem by replacing the value of the output variable with label $0$ if it was less than the total median or with label $1$ otherwise. The datasets contained $506$, $442$ and $569$ data points, respectively. However, in order to reliably evaluate the model performance on small data, in each dataset we sampled data subsets of length $N$ (from $N = 12$ to $N = 42$) at random. The model was tested by making a prediction for the $(N+1)$-th sample. The sampling and testing was repeated $500$ times for different re-samplings of the full dataset and the model performance scores were averaged. The model performance was compared to three other models: random forest, AdaBoost and decision trees \cite{freund1999short, breiman2001random, quinlan1986induction, dietterich2000ensemble}. Our experiments show that the Bayesian variational meta-classifier is competitive when compared to other ensemble approaches, achieving the best performances in diabetes and breast cancer datasets (Fig.~\ref{fig:meta-classifier}B). In the house pricing dataset the Bayesian variational meta-classifier has competitive performance when the training set is smaller than $20$ data-points, but the performance degrades for higher number of training samples. This decline in performance is likely to be caused by the limitations of our recurrent architecture. Note that the variational meta-classifier is applied to each dataset without any further training, while the other methods are trained separately on each dataset. 

\section{Conclusions}
In this paper we introduced a likelihood-free variational method based on the minimization of the forward KL divergence between the model joint distribution and a factorized variational joint distribution. We focused our exposition on variational marginalization problems where a Bayesian predictive distribution is obtained by marginalizing over a large space of latent variables. 

\bibliographystyle{unsrtnat}
\bibliography{reference}

\begin{thebibliography}{53}
\providecommand{\natexlab}[1]{#1}
\providecommand{\url}[1]{\texttt{#1}}
\expandafter\ifx\csname urlstyle\endcsname\relax
  \providecommand{\doi}[1]{doi: #1}\else
  \providecommand{\doi}{doi: \begingroup \urlstyle{rm}\Url}\fi

\bibitem[Jordan et~al.(1999)Jordan, Ghahramani, Jaakkola, and
  Saul]{jordan1999introduction}
M.~I. Jordan, Z.~Ghahramani, T.~S. Jaakkola, and L.~K. Saul.
\newblock An introduction to variational methods for graphical models.
\newblock \emph{Machine Learning}, 37\penalty0 (2):\penalty0 183--233, 1999.

\bibitem[Hoffman et~al.(2013)Hoffman, Blei, Wang, and
  Paisley]{hoffman2013stochastic}
M.~D. Hoffman, D.~M. Blei, C.~Wang, and J.~Paisley.
\newblock Stochastic variational inference.
\newblock \emph{The Journal of Machine Learning Research}, 14\penalty0
  (1):\penalty0 1303--1347, 2013.

\bibitem[Ranganath et~al.(2014)Ranganath, Gerrish, and
  Blei]{ranganath2014black}
R.~Ranganath, S.~Gerrish, and D.~Blei.
\newblock Black box variational inference.
\newblock \emph{International Conference on Artificial Intelligence and
  Statistics}, 2014.

\bibitem[Rezende et~al.(2014)Rezende, Mohamed, and
  Wierstra]{rezende2014stochastic}
D.~J. Rezende, S.~Mohamed, and D.~Wierstra.
\newblock Stochastic backpropagation and approximate inference in deep
  generative models.
\newblock \emph{International Conference on Machine Learning}, 2014.

\bibitem[Kucukelbir et~al.(2017)Kucukelbir, Tran, Ranganath, Gelman, and
  Blei]{kucukelbir2017automatic}
A.~Kucukelbir, D.~Tran, R.~Ranganath, A.~Gelman, and D.~M. Blei.
\newblock Automatic differentiation variational inference.
\newblock \emph{The Journal of Machine Learning Research}, 18\penalty0
  (1):\penalty0 430--474, 2017.

\bibitem[Tran et~al.(2016)Tran, Kucukelbir, Dieng, Rudolph, Liang, and
  Blei]{tran2016edward}
D.~Tran, A.~Kucukelbir, A.~B. Dieng, M.~Rudolph, D.~Liang, and D.~M. Blei.
\newblock Edward: A library for probabilistic modeling, inference, and
  criticism.
\newblock \emph{arXiv preprint arXiv:1610.09787}, 2016.

\bibitem[Kingma and Welling(2013)]{kingma2013auto}
D.~P. Kingma and M.~Welling.
\newblock Auto-encoding variational {B}ayes.
\newblock \emph{arXiv preprint arXiv:1312.6114}, 2013.

\bibitem[Husz{\'a}r(2017)]{huszar2017variational}
F.~Husz{\'a}r.
\newblock Variational inference using implicit distributions.
\newblock \emph{arXiv preprint arXiv:1702.08235}, 2017.

\bibitem[Ritchie et~al.(2016)Ritchie, Horsfall, and Goodman]{ritchie2016deep}
D.~Ritchie, P.~Horsfall, and N.~D. Goodman.
\newblock Deep amortized inference for probabilistic programs.
\newblock \emph{arXiv preprint arXiv:1610.05735}, 2016.

\bibitem[Dumoulin et~al.(2017)Dumoulin, Belghazi, Poole, Mastropietro, Lamb,
  Arjovsky, and Courville]{dumoulin2016adversarially}
V.~Dumoulin, I.~Belghazi, B.~Poole, O.~Mastropietro, A.~Lamb, M.~Arjovsky, and
  A.~Courville.
\newblock Adversarially learned inference.
\newblock \emph{International Conference on Learning Representations}, 2017.

\bibitem[Gneiting and Raftery(2005)]{gneiting2005weather}
T.~Gneiting and A.~E. Raftery.
\newblock Weather forecasting with ensemble methods.
\newblock \emph{Science}, 310\penalty0 (5746):\penalty0 248--249, 2005.

\bibitem[Prokhorov et~al.(2002)Prokhorov, Feldkarnp, and
  Tyukin]{prokhorov2002adaptive}
D.~V. Prokhorov, L.~A. Feldkarnp, and I.~Y. Tyukin.
\newblock Adaptive behavior with fixed weights in {RNN}: an overview.
\newblock \emph{International Joint Conference on Neural Networks}, 3, 2002.

\bibitem[Santoro et~al.(2016)Santoro, Bartunov, Botvinick, Wierstra, and
  Lillicrap]{santoro2016meta}
A.~Santoro, S.~Bartunov, M.~Botvinick, D.~Wierstra, and T.~Lillicrap.
\newblock Meta--learning with memory--augmented neural networks.
\newblock \emph{International Conference on Machine Learning}, 2016.

\bibitem[Vinyals et~al.(2016)Vinyals, Blundell, Lillicrap, and
  Wierstra]{vinyals2016matching}
O.~Vinyals, C.~Blundell, T.~Lillicrap, and D.~Wierstra.
\newblock Matching networks for one shot learning.
\newblock \emph{Advances in Neural Information Processing Systems}, 2016.

\bibitem[Bishop(2006)]{christopher2016pattern}
C.~M. Bishop.
\newblock \emph{Pattern Recognition and Machine Learning}.
\newblock Springer, 2006.

\bibitem[Minka(2001)]{minka2001expectation}
T.~P. Minka.
\newblock Expectation propagation for approximate {B}ayesian inference.
\newblock \emph{Uncertainty in Artificial Intelligence}, 2001.

\bibitem[Barthelm{\'e} and Chopin(2011)]{barthelme2011abc}
S.~Barthelm{\'e} and N.~Chopin.
\newblock {ABC-EP}: Expectation propagation for likelihood-free {B}ayesian
  computation.
\newblock \emph{Internetional Conference on Machine Learning}, pages 289--296,
  2011.

\bibitem[Tavar{\'e} et~al.(1997)Tavar{\'e}, Balding, Griffiths, and
  Donnelly]{tavare1997inferring}
S.~Tavar{\'e}, D.~J. Balding, R.~C. Griffiths, and P.~Donnelly.
\newblock Inferring coalescence times from {DNA} sequence data.
\newblock \emph{Genetics}, 145\penalty0 (2):\penalty0 505--518, 1997.

\bibitem[Pritchard et~al.(1999)Pritchard, Seielstad, Perez-Lezaun, and
  Feldman]{pritchard1999population}
J.~K. Pritchard, M.~T. Seielstad, A.~Perez-Lezaun, and M.~W. Feldman.
\newblock Population growth of human y chromosomes: a study of {Y} chromosome
  microsatellites.
\newblock \emph{Molecular Biology and Evolution}, 16\penalty0 (12):\penalty0
  1791--1798, 1999.

\bibitem[Tran et~al.(2017{\natexlab{a}})Tran, Nott, and
  Kohn]{tran2017variational}
M.~N. Tran, D.~J. Nott, and R.~Kohn.
\newblock Variational {B}ayes with intractable likelihood.
\newblock \emph{Journal of Computational and Graphical Statistics}, 26\penalty0
  (4):\penalty0 873--882, 2017{\natexlab{a}}.

\bibitem[S. et~al.(2018)S., S., and Z.]{shi2018kernel}
Jiaxin S., Shengyang S., and Jun Z.
\newblock Kernel implicit variational inference.
\newblock \emph{International Conference on Learning Representations}, 2018.

\bibitem[Blum and Fran{\c{c}}ois(2010)]{blum2010non}
M.~G.~B. Blum and O.~Fran{\c{c}}ois.
\newblock Non-linear regression models for approximate {B}ayesian computation.
\newblock \emph{Statistics and Computing}, 20\penalty0 (1):\penalty0 63--73,
  2010.

\bibitem[Papamakarios and Murray(2016)]{papamakarios2016fast}
G.~Papamakarios and I.~Murray.
\newblock Fast epsilon-free inference of simulation models with {B}ayesian
  conditional density estimation.
\newblock \emph{Advances in Neural Information Processing Systems}, pages
  1028--1036, 2016.

\bibitem[Le et~al.(2017)Le, Baydin, Zinkov, and Wood]{le2017using}
T.~A. Le, A.~G. Baydin, R.~Zinkov, and F.~Wood.
\newblock Using synthetic data to train neural networks is model-based
  reasoning.
\newblock \emph{International Joint Conference on Neural Networks}, 2017.

\bibitem[Jaderberg et~al.(2014)Jaderberg, Simonyan, Vedaldi, and
  Zisserman]{jaderberg2014synthetic}
M.~Jaderberg, K.~Simonyan, A.~Vedaldi, and A.~Zisserman.
\newblock Synthetic data and artificial neural networks for natural scene text
  recognition.
\newblock \emph{arXiv preprint arXiv:1406.2227}, 2014.

\bibitem[Jaderberg et~al.(2016)Jaderberg, Simonyan, Vedaldi, and
  Zisserman]{jaderberg2016reading}
M.~Jaderberg, K.~Simonyan, A.~Vedaldi, and A.~Zisserman.
\newblock Reading text in the wild with convolutional neural networks.
\newblock \emph{International Journal of Computer Vision}, 116\penalty0
  (1):\penalty0 1--20, 2016.

\bibitem[Gupta et~al.(2016)Gupta, Vedaldi, and Zisserman]{gupta2016synthetic}
A.~Gupta, A.~Vedaldi, and A.~Zisserman.
\newblock Synthetic data for text localisation in natural images.
\newblock \emph{Proceedings of the IEEE Conference on Computer Vision and
  Pattern Recognition}, pages 2315--2324, 2016.

\bibitem[Stark et~al.(2015)Stark, Haz{\i}rbas, Triebel, and
  Cremers]{stark2015captcha}
F.~Stark, C.~Haz{\i}rbas, R.~Triebel, and D.~Cremers.
\newblock Captcha recognition with active deep learning.
\newblock \emph{Workshop New Challenges in Neural Computation}, page~94, 2015.

\bibitem[G{\"u}{\c{c}}l{\"u}t{\"u}rk et~al.(2016)G{\"u}{\c{c}}l{\"u}t{\"u}rk,
  G{\"u}{\c{c}}l{\"u}, van Lier, and van Gerven]{guccluturk2016convolutional}
Y.~G{\"u}{\c{c}}l{\"u}t{\"u}rk, U.~G{\"u}{\c{c}}l{\"u}, R.~van Lier, and
  M.~A.~J. van Gerven.
\newblock Convolutional sketch inversion.
\newblock \emph{European Conference on Computer Vision}, pages 810--824, 2016.

\bibitem[Ambrogioni et~al.(2017{\natexlab{a}})Ambrogioni, G{\"u}{\c{c}}l{\"u},
  Maris, and van Gerven]{ambrogioni2017estimating}
L.~Ambrogioni, U.~G{\"u}{\c{c}}l{\"u}, E.~Maris, and M.~van Gerven.
\newblock Estimating nonlinear dynamics with the {ConvNet} smoother.
\newblock \emph{arXiv preprint arXiv:1702.05243}, 2017{\natexlab{a}}.

\bibitem[Mescheder et~al.(2017)Mescheder, Nowozin, and
  Geiger]{mescheder2017adversarial}
L.~Mescheder, S.~Nowozin, and A.~Geiger.
\newblock Adversarial variational {B}ayes: {U}nifying variational autoencoders
  and generative adversarial networks.
\newblock \emph{arXiv preprint arXiv:1701.04722}, 2017.

\bibitem[Tran et~al.(2017{\natexlab{b}})Tran, Ranganath, and
  Blei]{tran2017hierarchical}
D.~Tran, R.~Ranganath, and David~M. Blei.
\newblock Hierarchical implicit models and likelihood-free variational
  inference.
\newblock \emph{arXiv preprint arXiv:1702.08896}, 2017{\natexlab{b}}.

\bibitem[Goodfellow et~al.(2014)Goodfellow, Pouget-Abadie, Mirza, Xu,
  Warde-Farley, Ozair, Courville, and Bengio]{goodfellow2014generative}
I.~Goodfellow, J.~Pouget-Abadie, M.~Mirza, B.~Xu, D.~Warde-Farley, S.~Ozair,
  A.~Courville, and Y.~Bengio.
\newblock Generative adversarial nets.
\newblock \emph{Advances in Neural Information Processing Systems}, pages
  2672--2680, 2014.

\bibitem[A. et~al.(2018)A., R., and Z.]{arora2018do}
Sanjeev A., Andrej R., and Yi~Z.
\newblock Do {GANs} learn the distribution? {S}ome theory and empirics.
\newblock \emph{International Conference on Learning Representations}, 2018.

\bibitem[Arora et~al.(2017)Arora, Ge, Liang, Ma, and
  Zhang]{arora2017generalization}
S.~Arora, R.~Ge, Y.~Liang, T.~Ma, and Y.~Zhang.
\newblock Generalization and equilibrium in generative adversarial nets
  ({GANs}).
\newblock \emph{arXiv preprint arXiv:1703.00573}, 2017.

\bibitem[Csill{\'e}ry et~al.(2010)Csill{\'e}ry, Blum, Gaggiotti, and
  Fran{\c{c}}ois]{csillery2010approximate}
K.~Csill{\'e}ry, M.~G.~B. Blum, O.~E. Gaggiotti, and O.~Fran{\c{c}}ois.
\newblock Approximate {B}ayesian computation ({ABC}) in practice.
\newblock \emph{Trends in Ecology \& Evolution}, 25\penalty0 (7):\penalty0
  410--418, 2010.

\bibitem[Marin et~al.(2012)Marin, Pudlo, Robert, and
  Ryder]{marin2012approximate}
J.~M. Marin, P.~Pudlo, C.~P. Robert, and R.~J. Ryder.
\newblock Approximate {B}ayesian computational methods.
\newblock \emph{Statistics and Computing}, 22\penalty0 (6):\penalty0
  1167--1180, 2012.

\bibitem[Murphy(2012)]{robert2014machine}
K.~P. Murphy.
\newblock \emph{Machine Learning, A Probabilistic Perspective}.
\newblock The MIT press, 2012.

\bibitem[Kingma and Ba(2014)]{kingma2014adam}
D.~P. Kingma and J.~Ba.
\newblock Adam: {A} method for stochastic optimization.
\newblock \emph{arXiv preprint arXiv:1412.6980}, 2014.

\bibitem[West(1996)]{west1996bayesian}
M.~West.
\newblock \emph{{B}ayesian {F}orecasting}.
\newblock Wiley Online Library, 1996.

\bibitem[Lorenz(1963)]{lorenz1963deterministic}
E.~N. Lorenz.
\newblock Deterministic nonperiodic flow.
\newblock \emph{Journal of the Atmospheric Sciences}, 20\penalty0 (2):\penalty0
  130--141, 1963.

\bibitem[Yu and Koltun(2015)]{yu2015multi}
F.~Yu and V.~Koltun.
\newblock Multi-scale context aggregation by dilated convolutions.
\newblock \emph{arXiv preprint arXiv:1511.07122}, 2015.

\bibitem[Ambrogioni et~al.(2017{\natexlab{b}})Ambrogioni, G{\"u}{\c{c}}l{\"u},
  van Gerven, and Maris]{ambrogioni2017kernel}
L.~Ambrogioni, U.~G{\"u}{\c{c}}l{\"u}, M.~A.~J. van Gerven, and E.~Maris.
\newblock The kernel mixture network: {A} nonparametric method for conditional
  density estimation of continuous random variables.
\newblock \emph{arXiv preprint arXiv:1705.07111}, 2017{\natexlab{b}}.

\bibitem[Evensen(2009)]{evensen2009data}
G.~Evensen.
\newblock \emph{{D}ata {A}ssimilation: {T}he {E}nsemble {K}alman {F}ilter}.
\newblock {S}pringer, 2009.

\bibitem[Mitchell(1997)]{mitchell1997machine}
T.~M. Mitchell.
\newblock \emph{Machine Learning}.
\newblock McGraw Hill, 1997.

\bibitem[Tokui et~al.(2015)Tokui, Oono, Hido, and Clayton]{tokui2015chainer}
S.~Tokui, K.~Oono, S.~Hido, and J.~Clayton.
\newblock Chainer: {A} next-generation open source framework for deep learning.
\newblock \emph{Workshop on Machine Learning Systems (NIPS)}, 5, 2015.

\bibitem[Harrison and Rubinfeld(1978)]{harrison1978hedonic}
David. Harrison and D.~L. Rubinfeld.
\newblock Hedonic housing prices and the demand for clean air.
\newblock \emph{Journal of Environmental Economics and Management}, 5\penalty0
  (1):\penalty0 81--102, 1978.

\bibitem[Efron et~al.(2004)Efron, Hastie, Johnstone, and
  Tibshirani]{efron2004least}
B.~Efron, T.~Hastie, I.~Johnstone, and R.~Tibshirani.
\newblock Least angle regression.
\newblock \emph{The Annals of Statistics}, 32\penalty0 (2):\penalty0 407--499,
  2004.

\bibitem[Street et~al.(1993)Street, Wolberg, and
  Mangasarian]{street1993nuclear}
W.~N. Street, W.~H. Wolberg, and O.~L. Mangasarian.
\newblock Nuclear feature extraction for breast tumor diagnosis.
\newblock \emph{{Biomedical Image Processing and Biomedical Visualization}},
  1905:\penalty0 861--871, 1993.

\bibitem[Freund et~al.(1999)Freund, Schapire, and Abe]{freund1999short}
Y.~Freund, R.~Schapire, and N.~Abe.
\newblock A short introduction to boosting.
\newblock \emph{Journal of the Japanese Society For Artificial Intelligence},
  14\penalty0 (771-780):\penalty0 1612, 1999.

\bibitem[Breiman(2001)]{breiman2001random}
L.~Breiman.
\newblock Random forests.
\newblock \emph{Machine Learning}, 45\penalty0 (1):\penalty0 5--32, 2001.

\bibitem[Quinlan(1986)]{quinlan1986induction}
J.~R. Quinlan.
\newblock Induction of decision trees.
\newblock \emph{Machine Learning}, 1\penalty0 (1):\penalty0 81--106, 1986.

\bibitem[Dietterich(2000)]{dietterich2000ensemble}
Thomas~G Dietterich.
\newblock Ensemble methods in machine learning.
\newblock \emph{International Workshop on Multiple Classifier Systems}, pages
  1--15, 2000.

\end{thebibliography}

\end{document}